\pdfoutput=1
\documentclass{article}

\usepackage[final]{corl_2026}

\usepackage{graphicx}
\usepackage[dvipsnames]{xcolor}

\usepackage{enumitem}
\usepackage{amsthm}
\usepackage{xspace}
\usepackage{array}
\usepackage{booktabs}
\usepackage{amssymb}
\usepackage{amsmath}
\usepackage{algorithm}
\usepackage{algpseudocode}

\newtheorem{theorem}{Proposition}

\definecolor{citecolor}{HTML}{3498DB}
\definecolor{urlcolor}{HTML}{485DFF}
\newcommand{\method}{\textsc{OmniGuide}\xspace}

\newenvironment{yhblock}{\begingroup\color{black}} {\endgroup}
\newcommand{\cmark}{\textcolor{ForestGreen}{\ensuremath{\boldsymbol{\checkmark}}}}
\newcommand{\xmark}{\textcolor{BrickRed}{\ensuremath{\boldsymbol{\times}}}}
\newcommand{\pmark}{\textcolor{BurntOrange}{\ensuremath{\blacktriangle}}}

\hypersetup{
    colorlinks=true,
    citecolor=citecolor,
    urlcolor=urlcolor,
    linkcolor=black,
}

\title{OmniGuide:
Universal Spatial Guidance Fields for Enhancing Generalist Robot Policies
}

\author{
Yunzhou Song\thanks{Equal contribution} \;\;
Long Le$^{*}$ \,
Yong-Hyun Park \,
Jie Wang \,
Junyao Shi \,
\\
\textbf{Lingjie Liu \,
Jiatao Gu \,
Eric Eaton \,
Dinesh Jayaraman \,
Kostas Daniilidis} \\
\\
University of Pennsylvania
}

\let\cite\citep

\begin{document}
\maketitle

\begin{abstract}
Vision-language-action (VLA) models have shown great promise as generalist policies for a large range of relatively simple tasks. However, they demonstrate limited performance on more complex tasks, such as those requiring complex spatial or semantic understanding, manipulation in clutter, or precise manipulation.
We propose \method, a flexible framework that improves VLA performance on such tasks by leveraging
versatile sources of guidance, such as 3D foundation models, semantic-reasoning VLMs, and human pose models. We show how
to naturally express many kinds of guidance as
differentiable energy functions with task-specific attractors and repellers located in 3D space, that influence the sampling of VLA actions.
In this way, \method enables complementary task-relevant guidance to improve a VLA model's performance on challenging tasks.
Extensive experiments in both simulation and real-world environments, across diverse sources of guidance like collision safety and human demonstration, prove that \method enhances the performance of state-of-the-art generalist policies (e.g., $\pi_{0.5}$, GR00T N1.6) significantly in terms of both success and safety rates. Critically, our unified framework matches or surpasses the performance of prior methods designed to incorporate \textit{specific} sources of guidance into VLA policies. Project website: \hyperlink{https://omniguide.github.io/}{https://omniguide.github.io/}

\end{abstract}


\begin{figure}[h!]
    \centering
    \vspace{-1em}
    \includegraphics[width=0.8\linewidth]{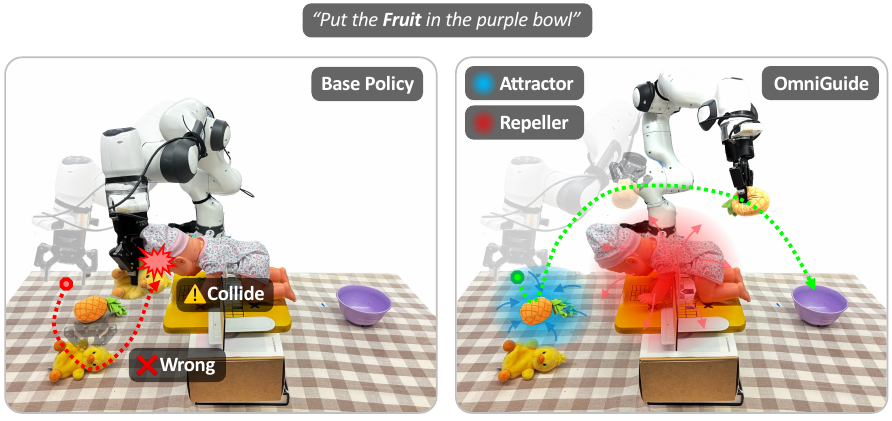}
    \vspace{-1em}
    \caption{\method unifies different sources of guidance via attractive and repulsive fields to improve the performance of generalist robot policies.}
    \label{fig:teaser}
\end{figure}

\section{Introduction}

\vspace{-0.5em}

Today's predominant paradigm for developing generalist robotic policies is to train Vision-Language-Action (VLA) models via Behavior Cloning (BC) on massive human teleoperated datasets \citep{open_x_embodiment_rt_x_2023, khazatsky2024droid, walke2023bridgedata}. While this approach has yielded capable models such as Gemini Robotics \cite{team2025gemini}, GR00T \cite{gr00tn1_2025}, MolmoAct \cite{lee2025molmoact}, and the $\pi$ series \cite{black2024pi0, intelligence2025pi05}, the scalability of pure imitation faces an inherent performance ceiling.
Without dense coverage of the space of all possible environments and tasks, which is practically impossible to achieve via human teleoperation alone, these VLAs pre-trained with behavior cloning often fall short of mastery. Instead of becoming experts, pre-trained VLAs often emerge as \emph{jacks-of-all-trades, masters of none}, possessing a broad understanding of diverse tasks but lacking the specialized precision required for reliable execution in a new environment. Consequently, even the most advanced VLA models often fail at the ``last mile" -- struggling with 3D collision avoidance, precise physical grounding, and robust articulated object manipulation \cite{gao2025taxonomy, zhang2025vlabench, zhang2025vla, hu2025vlsa}. Common approaches~\cite{tan2025interactive, li2025vla, li2025jarvis} to bridge this gap require intensive post-training and fine-tuning on additional high-quality robotic data in the target environment, which is both expensive and scarce.

We explore an alternative strategy to overcome these weaknesses of VLAs that relies on no additional robotic data, and no additional training. Instead, we propose that VLA models could simply ``get by with a little help from their friends'', namely, 3D geometry reconstruction, human pose estimation, Vision-Language Models (VLMs), or robot motion planners.
But how might these models help a VLA to overcome its deficiencies?
Our main insight is that 3D Euclidean space is the most suitable domain for expressing constraints such as collision avoidance, trajectory following, and goal reaching, as these constraints are inherently geometric. Perception foundation models either predict in 3D or can be naturally lifted to 3D.

We introduce \method, a general framework for improving VLA generalist policies through \textit{unified inference-time guidance}. \method applies to any VLA that generates actions through diffusion or flow matching.
Rather than attempting to bake every constraint into pre-training, \method leverages external foundation models, like 3D reconstruction \cite{wang2025vggt} for safety-aware collision avoidance, VLMs \cite{comanici2025gemini} for semantic guidance, and hand-tracking \cite{ye2025predicting} for one-shot human demonstrations, to inject knowledge missing in the VLA, and steer the policy’s actions at test-time.

We unify these diverse sources of guidance through a differentiable energy function over 3D space,
whose spatial gradients represent two types of ``forces'': attraction toward task-relevant regions such as goals, and repulsion away from undesirable regions such as collisions and other safety constraints.
These external gradients, in conjunction with the VLA’s native velocity field, narrow the VLA’s multi-modal solution space, steering the robot toward actions that are simultaneously task-effective, safe, and physically grounded, as shown in Fig.~\ref{fig:teaser}.

To validate our approach, we perform extensive experiments both in simulation and real-world across three classes of guidance and two generalist VLA base models \cite{intelligence2025pi05, gr00tn1_2025} on challenging manipulation tasks. \method consistently improves the base VLA models across different metrics, including increasing success rates from 28.7\% to 82.4\% and safety rates from 7.0\% to 81.0\%, all without incurring significant execution latencies or requiring retraining.
We summarize our main contributions here:
\vspace{-0.5em}
\begin{enumerate}[leftmargin=*]
\item We introduce 3D Cartesian space as a universal guidance frame. Unlike joint space, in 3D space, task objectives can be expressed independently of the policy architecture or robot morphology.
\item We introduce test-time guidance for any differentiable generative policy, agnostic of how it has been trained or of what particular model was implemented.
\item Our guidance term can express attractive (semantically grounded targets, demonstration trajectory) or repulsive (obstacle avoidance) energy fields, as well as their synergy. These energies and their gradients are computed in real-time, allowing test-time steering in changing environments.
\item Our framework is versatile: we demonstrate that it permits many foundation models, sometimes in tandem, to guide state-of-the-art VLA policies and improve their performance on challenging tasks.
We experimentally analyze the trade-offs between staying on the prior distribution vs.~being safe or following new targets.
\end{enumerate}

\section{Related Work}
\vspace{-0.5em}
\textbf{Generalist Robotic Manipulation Policies~~}
The current frontier in generalist robot manipulation is driven by Vision-Language-Action (VLA) models, which are large, multimodal systems that unify perception, reasoning, and control by coupling VLMs with generative models for actions. These end-to-end architectures ~\cite{rt22023arxiv, kim24openvla, octo_2023, black2024pi0,li2024cogact} allow robots to leverage the vast semantic knowledge embedded in VLM backbones. However, recent evaluations reveal several key limitations, including shortcomings in grounding abstract language to diverse scenarios~\cite{ team2025gemini, intelligence2025pi05, liang2025pixelvla}, long-term task planning ~\cite{shi2025memoryvla, jang2025contextvla, zhao2025cot}, and spatial perception and reasoning ~\cite{lee2025molmoact, qu2025spatialvla, zheng2024tracevla, spatialforcing2025}. While contemporary efforts attempt to mitigate these failures through supervised finetuning on curated datasets \cite{khazatsky2024droid,lee2025molmoact,intelligence2025pi05} or reinforcement learning \cite{chen2025tgrpo, intelligence2025pi} for broader action exploration, these approaches incur prohibitive computational costs and risk degrading the capabilities of the base model via catastrophic forgetting.

\textbf{Test-Time Guidance for Generative Policies~~}
Test-time guidance has emerged as a powerful paradigm for steering the generative process of diffusion~\cite{pmlr-v283-chatzipantazis25a, park2025demodiffusion, wagenmaker2025steering, wang2025inference, sun2025latent, du2025dynaguide, carvalho2025motion} or flow matching~\cite{black2025real, dai2025safe} policies, allowing the incorporation of extra constraints or objectives beyond those initially trained. However, previous works are either limited to modifying only the initial noise distribution~\cite{park2025demodiffusion, wagenmaker2025steering, yang2025steering,kwok2025robomonkey} or rely on small, specialized models trained on narrow domains, such as the latent space or dynamic predictive models~\cite{du2025dynaguide, sun2025latent, gupta2025umi, nakamoto2024steering, luo2025concept}, which hinder their utility in more general in-the-wild robotic settings. Recent approaches introduced safety constraints to flow matching:
SafeFlow \cite{dai2025safe} applies CBF-based quadratic optimization at each integration step, whereas SafeFlowMatcher \cite{yang2025safeflowmatcher} enforces safety through a two-stage prediction-correction integrator.
None of them is tested in complex environments and tasks or in combination with any generalist policy.
The closest work to ours is Inference-time Policy Steering~\cite{wang2025inference}. While their method is limited to expertise-demanding kinesthetic teaching and requires
{iterative MCMC sampling process}, \method demonstrates a unified guidance framework over several guidance types {including both the initial noise prior and intermediate conditional distributions along the generative trajectory}, and leverages a hand-tracking model for ease of use.

\textbf{Harnessing Foundation Priors for Robot Manipulation~~}
While modern robot manipulation systems are increasingly powered by large Vision-Language Models (VLMs) for task specification and generalization~\cite{shi2025hi, dalal2025local,shi2025maestro, huang2024rekep, liu2026vls}, most other foundation models with rich physical or geometric understanding remain underexploited in robotic control. Meanwhile, a growing ecosystem of specialized foundation models, such as 3D reconstruction and geometry-aware models ~\cite{wang2025vggt, song2025avi, liu2025geometry}, articulation and kinematic reasoning models~\cite{le2024articulate, le2025pixie, mandi2024real2code, liu2024singapo}, human motion~\cite{pavlakos2024reconstructing, ye2025predicting}, and object tracking systems~\cite{karaev2025cotracker3, song2024track, wen2024foundationpose} offer powerful priors. Yet, it remains unclear how to harness these heterogeneous models to improve robotic policies. \method addresses this challenge by incorporating these diverse foundation models, ranging from VLMs to 3D geometry and human priors, through a unified inference-time guidance mechanism without having to retrain the base VLAs.

\begin{yhblock}

\section{Background}
\vspace{-1em}
\label{subsec:background}
We present a unified inference-time guidance framework for flow-matching-based VLA policies. Our key idea is to treat diverse sources of task objectives
as \emph{energy functions defined over Cartesian trajectories}. The energies steer the generative flow of a pretrained VLA policy during inference, without retraining or collecting additional data.

\noindent\textbf{Action-Chunk Flow Matching as a Generative Prior.~~}
We consider a pretrained VLA policy that generates an action chunk \(\mathbf{A}_t=[\mathbf{a}_t,\mathbf{a}_{t+1},\dots,\mathbf{a}_{t+H-1}]\) conditioned on the current observation \(\mathbf{o}_t\). For simplicity, we drop the environment time subscript when clear. Depending on the policy, actions may live in joint space or a learned latent space: GR00T~\citep{gr00tn1_2025} operates in a learned latent action space, whereas \(\pi_{0.5}\)~\citep{intelligence2025pi05} uses action chunks in the robot joint space.

The VLA policy is trained using flow matching~\citep{lipman2022flow,liu2022flow} to model the conditional distribution \(p(\mathbf{A}\mid\mathbf{o})\). Denoting the denoising time by \(\tau\), the model learns a time-dependent vector field
\begin{equation}
\frac{d \mathbf{A}^{\tau}}{d\tau}=\mathbf{v}_{\theta}(\mathbf{A}^{\tau},\mathbf{o}),
\qquad \tau\in[0,1],
\label{eq:flow}
\end{equation}
where \(\tau=0\) corresponds to noise and \(\tau=1\) to clean samples. At inference time, the base policy integrates this field, e.g.,
\begin{equation}
\mathbf{A}^{\tau+\delta}
=
\mathbf{A}^{\tau}
+
\delta\,\mathbf{v}_{\theta}(\mathbf{A}^{\tau},\mathbf{o}).
\label{eq:denoise}
\end{equation}
This process provides a stochastic, multimodal prior over action chunks. Tweedie's formula estimates the posterior mean of the clean action trajectory at denoising step \(\tau\) as~\citep{efron2011tweedie,chung2022diffusion,kim2025flowdps}
\begin{equation}
\tilde{\mathbf{A}}^{\tau}
=
\mathbb{E}[\mathbf{A}^{1}\mid\mathbf{A}^{\tau}]
=
\mathbf{A}^{\tau}
+
(1-\tau)\mathbf{v}_{\theta}(\mathbf{A}^{\tau},\mathbf{o}).
\label{eq:clean}
\end{equation}
\noindent\textbf{Guidance for Controllable Generation.~~}
The pretrained VLA defines a \emph{naturalness prior} over actions, capturing semantic understanding, smoothness, kinematic feasibility, and contact patterns present in the training data. However, this prior alone does not explicitly enforce objectives such as collision avoidance or semantic grounding. Rather than interpreting an objective \(\mathbf{y}\) only as a discrete conditioning variable, we adopt an energy-based view~\citep{liu2022compositional,du2023reduce} and define the task-induced energy
\begin{equation}
\mathcal{L}_{\mathbf{y}}(\tilde{\mathbf{A}}^{\tau})
:=
-\log p(\mathbf{y}\mid\tilde{\mathbf{A}}^{\tau}).
\label{eq:energy}
\end{equation}
Guided flow matching can then be interpreted as \emph{energy shaping}: The generative process follows the superposition of the learned naturalness field and the gradient of a task-defined energy. In practice, this yields the guided denoising update
\begin{equation}
\mathbf{A}^{\tau+\delta}
=
\mathbf{A}^{\tau}
+
\delta\Big(
\mathbf{v}_{\theta}(\mathbf{A}^{\tau},\mathbf{o})
-
\lambda\,\mathrm{clip}(
\nabla_{\mathbf{A}^{\tau}}\mathcal{L}_{\mathbf{y}}(\mathbf{A}^{\tau}),\alpha)
\Big),
\label{eqn:intermediate}
\end{equation}
where \(\lambda\) controls guidance strength, which we keep constant across timesteps for stability~\citep{patel2024steering}, and \(\alpha\) clips gradients for numerical stability. The derivations for diffusion policies are similar~\citep{gao2025diffusionmeetsflow,lai2025principles}.

\begin{figure}
    \centering
    \includegraphics[width=1.0\linewidth]{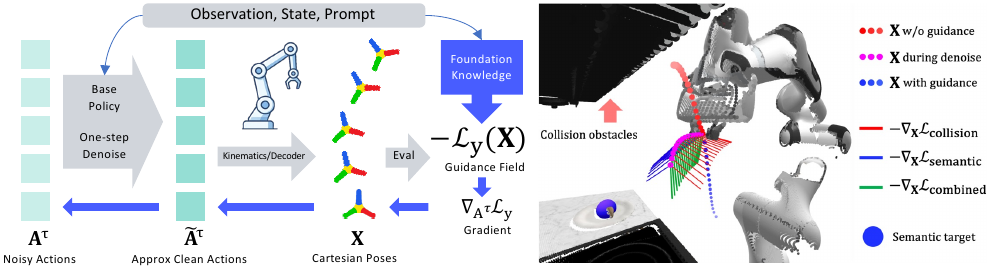}
    \vspace{-1em}
    \caption{\textbf{Left: Method Overview.} For each denoising step, \method first estimates the clean action chunk \(\tilde{\mathbf{A}}^{\tau}\) using the base policy \(\mathbf{v}_{\theta}\), and maps it through differentiable kinematic models or decoders to Cartesian trajectories \(\mathbf{X}\). Then they are evaluated by task energy functions \(\mathcal{L}_{\mathbf{y}}\) constructed from foundation models, whose gradients are backpropagated to the noisy actions \(\mathbf{A}^{\tau}\). \textbf{Right: Guidance Visualization.} Repulsive collision energies push predicted trajectories away from obstacles, while attractive semantic energies pull the gripper toward grounded targets; multiple guidance sources naturally blend by summing energies.}
    \label{fig:method_vis}
\end{figure}

\section{OmniGuide: Universal Guidance Fields for Enhancing Generalist Policies}
\vspace{-0.5em}
Having set the background above, we are now ready to describe our approach to improve a generalist policy for different kinds of auxiliary objectives by guidance from heterogeneous sources.

\vspace{-0.5em}
\subsection{Guidance Defined in Cartesian Space}
\label{sec: Cartesian}
\vspace{-0.5em}
Most task constraints of interest are naturally defined in Cartesian space rather than in action space, and they apply to the \emph{clean} trajectory that the policy will ultimately execute. Therefore, guidance cannot be evaluated directly on noisy actions \(\mathbf{A}^{\tau}\). Instead, at each denoising step, \method first estimates the clean action chunk \(\tilde{\mathbf{A}}^{\tau}\) using Eq.~\ref{eq:clean}, decodes it from latent to robot action space if required, and predicts the corresponding Cartesian trajectory $\mathbf{X}$ through a differentiable dynamics and kinematics model $f$~\citep{liu2024differentiablerobotrendering}:
\begin{equation}
\mathbf{X}=f(\mathbf{p}\mid\tilde{\mathbf{A}}^{\tau},\mathbf{s}),
\label{eq:dynamics}
\end{equation}
where \(\mathbf{p}\) denotes probe points or frames rigidly attached to the robot, such as the end-effector or body keypoints, and \(\mathbf{s}\) is the current robot state. We then evaluate \(\mathcal{L}_{\mathbf{y}}(\mathbf{X})\) in Cartesian space and backpropagate through the robot model to obtain the posterior guidance term on the noisy action iterate, \(\nabla_{\mathbf{A}^{\tau}}\log p(\mathbf{y}\mid\mathbf{A}^{\tau})=-\nabla_{\mathbf{A}^{\tau}}\mathcal{L}_{\mathbf{y}}(\mathbf{X})\). (Fig.~\ref{fig:method_vis} Left)

In addition to intermediate denoising guidance, we optionally apply guidance at the initial noise distribution. Motivated by recent findings that combining initial-distribution guidance with denoising guidance improves performance~\citep{yoon2025psi}, we sample \(N\) candidates \(\{\mathbf{A}_{i}^{0}\}_{i=1}^{N}\), approximately denoise them, and select the candidate whose clean trajectory has the lowest energy. This is a Monte-Carlo approximation to an energy-reweighted prior~\citep{mackay2003information},
\begin{equation}
p^{*}(\mathbf{A}^{0})
\propto
p(\mathbf{A}^{0})\exp(-\mathcal{L}_{\mathbf{y}}(\mathbf{A}^{1})/\alpha).
\label{eqn:init}
\end{equation}
\subsection{A Universal Spatial Guidance Field}
\label{sec:costf}
\vspace{-0.5em}
All guidance sources in our framework instantiate the same abstraction: they define \emph{attractive or repulsive energy fields in 3D space} that act on the predicted Cartesian trajectory (Fig.~\ref{fig:method_vis} Right). These fields are projected to action space via differentiable kinematics or decoder as described in Sec.~\ref{sec: Cartesian} and shape the VLA's generative process.
As the guidance is unified as energies $\mathcal{L}_{\mathbf{y}}(\mathbf{x})=-\log p_{\mathbf{y}}(\mathbf{x})$, it is easy to introduce guidance from multimodal sources, and even combine them in a shared space.

\noindent\textbf{Collision Avoidance via Repulsive Fields.~~} \label{sec:collision}
We construct this objective from an environment point cloud \(\{\mathbf{p}\}\), built either from RGB-D images or calibrated VGGT predictions. We remove task-relevant regions such as the object of interest and the robot body using CLIP features~\citep{radford2021learning} and mesh rendering, then quantize the remaining point cloud into an occupancy grid $O\in\mathbb{R}^{H\times W \times L}$ and compute a discrete signed distance function \(\mathrm{SDF}_{O}(\mathbf{x})\). Under the surface-reconstruction assumption that query points inside occupied volumes are rarely observed (Appendix~\ref{app:sdf_assumption}), we model collision avoidance within a bounded risk region ${\Omega} = \{\mathbf{x} \in \mathbb{R}^3 \mid 0 < \mathrm{SDF}_O(\mathbf{x}) \le d\}$, where $d$ is a pre-defined barrier distance. This yields the repulsive energy
\begin{equation}
\mathcal{L}_{C}(\mathbf{x})
=
-\log p_{C}(\mathbf{x})
=
-\log \mathrm{SDF}_{O}(\mathbf{x})
+\mathrm{Const}.
\label{eq:L_collision}
\end{equation}
By the definition of the SDF, the gradient with respect to the query position $\mathbf{x}$ takes the form
$
-\nabla_\mathbf{x} \mathcal{L}_C(\mathbf{x}) = \frac{1}{\mathrm{SDF}_O(\mathbf{x})} \frac{\mathbf{x}-\mathbf{p}^*}{\|\mathbf{x}-\mathbf{p}^*\|},
$
where $\mathbf{p}^*$ is the closest point on the obstacle surface to $\mathbf{x}$.
Thus, the induced gradient pushes trajectories away from obstacles with strength increasing near contact. In implementation, we compute SDF gradients by finite differences with trilinear interpolation.

\noindent\textbf{Semantic Grounding via Attractive Targets.~~}
To augment the semantic grounding capabilities of the VLA base policy, we make use of the zero-shot spatial reasoning of state-of-the-art VLMs. We query the VLM with the current image observations and task instruction to localize the most task-relevant object or affordance in pixel space, denoted \(\mathbf{u}^{*}\in\mathbb{R}^{2}\). Using the aligned depth map, we back-project this point into the 3D space to obtain a target centroid \(\mathbf{x}^{*}\in\mathbb{R}^{3}\). Under a Gaussian grounding assumption (Appendix~\ref{app:gaussian_guidance}), \(p_{S}(\mathbf{x})=\mathcal{N}(\mathbf{x}\mid\mathbf{x}^{*},\sigma_{S}^{2}\mathbf{I})\), and the attractive energy is
\begin{equation}
\mathcal{L}_{S}(\mathbf{x})
=
-\log p_{S}(\mathbf{x})
=
\frac{\|\mathbf{x}-\mathbf{x}^{*}\|_{2}^{2}}{2\sigma_{S}^{2}}
+\mathrm{Const}.
\label{eq:L_s}
\end{equation}
Complementary to this positional guidance, we formulate an orientation objective that aligns the gripper's forward vector with the approach direction \(\mathbf{x}^{*}-\mathbf{x}\); details are provided in the Appendix.

\noindent\textbf{Demonstration Trajectories as Sequential Attractors.~~}
Direct retargeting is complicated by kinematic differences between human and robot morphology~\citep{shi2025zeromimic}. Meanwhile, robot motion planners can generate feasible and safe trajectories given goals and the environment. We leverage HaPTIC~\citep{ye2025predicting} to extract human hand poses from a one-shot demonstration, and cuRobo~\cite{curobo_report23} to plan safe robot motion, yielding reference trajectories \(\mathcal{H}=\{\mathbf{h}_{0},\dots,\mathbf{h}_{N-1}\}\). To align the robot's predicted end-effector trajectory \(\mathcal{X}=\{\mathbf{x}_{t}\}_{t=0}^{H-1}\) with their references, we use the monotonic matching strategy in Appendix Sec.~\ref{sec: matching}, Fig.~\ref{fig:trajguide}.
For each matched correspondence, we model \(p_{H}(\mathbf{x}\mid\mathbf{h}^{*})=\mathcal{N}(\mathbf{x}\mid\mathbf{h}^{*},\sigma_{H}^{2}\mathbf{I})\), yielding the trajectory-level attraction energy, where $\mathcal{M}$ is the matched set.
\begin{equation}
\mathcal{L}_{H}(\mathcal{X})
=
\frac{1}{2\sigma_{H}^{2}}
\sum_{(\mathbf{x}_{i},\mathbf{h}_{i}^{*})\in\mathcal{M}}
\|\mathbf{x}_{i}-\mathbf{h}_{i}^{*}\|_{2}^{2}
+\mathrm{Const}.
\label{eq:L_H}
\end{equation}
The guidance fields we introduced are deterministic and task-specific, and may exhibit local minima or incomplete modeling of contact dynamics. The pretrained VLA policy, by contrast, provides a stochastic and diverse prior over action chunks (Appendix Fig.~\ref{fig:localmin}). Guided flow matching combines these complementary strengths: the base policy supplies diversity and realism, while guidance shapes generation toward task-relevant regions of the action space.
\end{yhblock}

\begin{figure}
    \centering
    \includegraphics[width=1\linewidth]{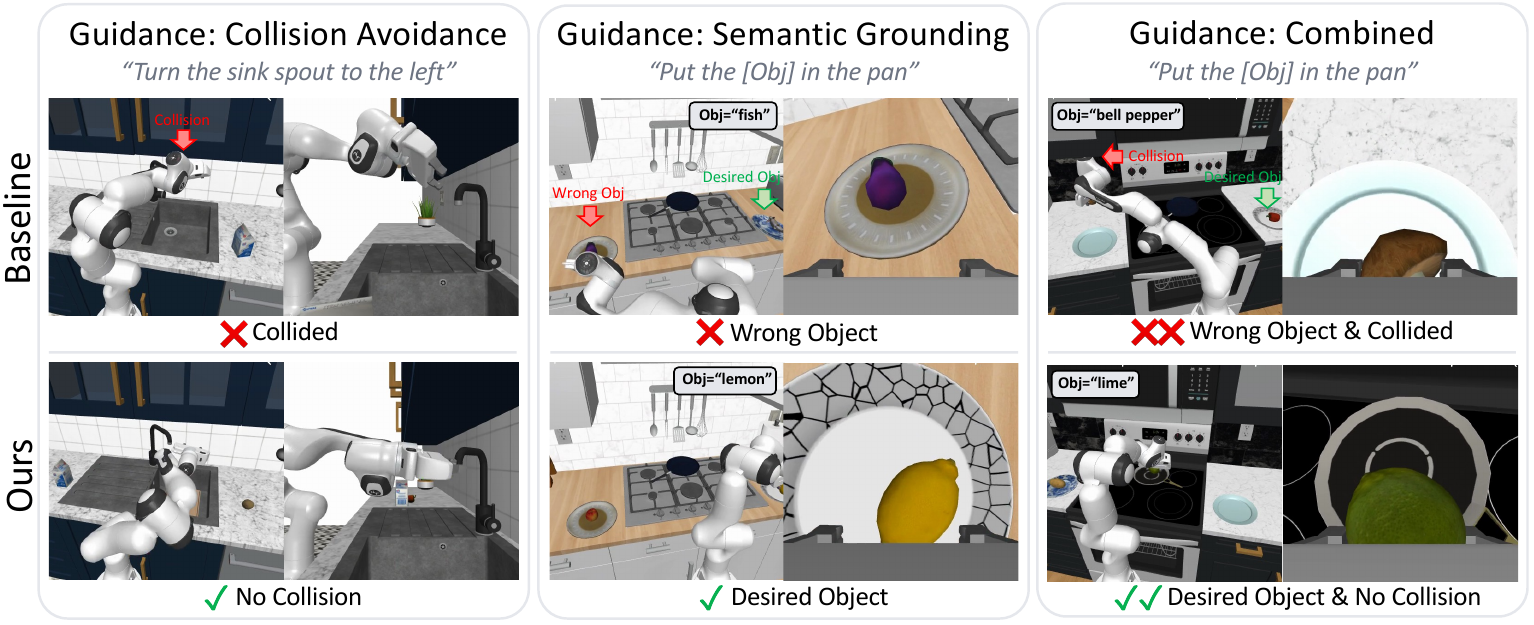}
    \vspace{-2em}
    \caption{\textbf{Simulation Experiments Qualitative Results.}
    \textbf{Left:} collision avoidance guidance for the \textit{TurnSinkSpout} task. \textbf{Middle:} semantic grounding guidance for the \textit{Multi-Choice} task. \textbf{Right:} combination of the two guidance for the \textit{Multi-Choice} task in a cluttered scene.
    }
    \label{fig:sim_vis}
    \vspace{-0.5em}
\end{figure}

\begin{figure}
    \centering
    \includegraphics[width=1.0\linewidth]{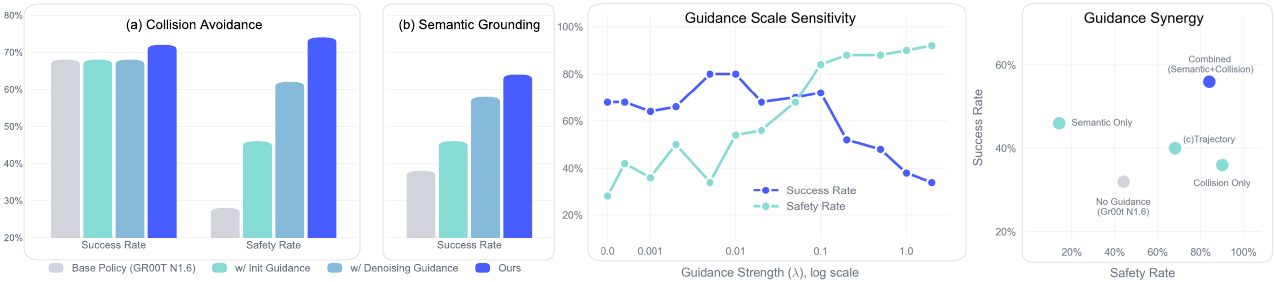}
    \vspace{-2em}
    \caption{\textbf{Simulation Experiments Quantitative Results.} \textbf{Left:} results for the collision avoidance and semantic tasks. \textbf{Mid:} sensitivity of guidance strength. \textbf{Right:} guidance synergy analysis.
    }
    \vspace{-1.5em}
    \label{fig:sim_all}
\end{figure}

\section{Experiments}
\vspace{-1em}
The goal of our experiments is to understand when, why, and how \method improves generalist Vision-Language-Action (VLA) policies. While \method is designed as a general and modular framework, its effectiveness depends on interacting factors: the strength of the pretrained VLA prior, the nature of the guidance signal, the trade-off between naturalness and specific objective satisfaction, and the ability to compose multiple heterogeneous objectives.

\vspace{-0.5em}
\subsection{Simulation Experiments}

\vspace{-0.5em}
We evaluate our unified guidance framework in a high-fidelity simulation using the RoboCasa benchmark~\cite{robocasa2024}. Our system is built upon the NVIDIA GR00T N1.6-3B foundation model, which serves as the backbone policy. It predicts relative Cartesian end-effector motions for a Franka Emika Panda robot, and performs the denoising in latent space. To study the synergy between guidance and generative priors, we exclude perception artifacts from foundation models by using ground-truth depth and object poses provided by the simulator.

\newcommand{\CA}{\textbf{\textcolor{red}{CA}}}
\newcommand{\SG}{\textbf{\textcolor{orange}{SG}}}
\newcommand{\TJ}{\textbf{\textcolor{Maroon}{TJ}}}

\noindent\textbf{Guidance Tasks:} We curate a challenging suite of tasks to test \method with distinct guidance modalities. For \textbf{(a) collision avoidance (\CA)}, we select two RoboCasa tasks with substantial clutter: pick-and-place and turning the sink spout (
see~Fig.~\ref{fig:sim_vis}). Note that we do not alter the original definition of success for these RoboCasa tasks. Thus, collision avoidance is only incidentally related to ``success'', and policies that incur frequent collisions can still have high ``success rates''. To directly measure collision avoidance behavior, we also track ``safety rates'' for each policy, where an episode is considered unsafe if at least one collision occurs between the robot body and static furniture.
For \textbf{(b) semantic grounding (\SG)}, we construct a \textit{Multi-Choice} task following \citep{li2025hamster, cheang2025gr}, where the agent must identify and manipulate a specific target object from multiple candidates initialized at randomized locations (see Fig.~\ref{fig:sim_vis} for an example). We perform the planner \textbf{(c) trajectory (\TJ)} guidance on the cluttered multi-choice task. Appendix Sec~\ref{sec: sim_setup} details the hyperparameter configuration.

\noindent\textbf{Results:} Fig.~\ref{fig:sim_vis} shows the experiment setup and qualitative results. All experiments are conducted over 50 independent trials, and we report the Success Rate and Safety Rate.
As detailed in Fig.~\ref{fig:sim_all}, our method significantly outperforms the unguided baseline across all settings. For \CA, as expected, it dramatically boosts the safety rate, and even slightly improves success rates. On \SG, \method's guidance is directly relevant to task success. \TJ~ guidance improves both metrics simultaneously.

\vspace{-0.5em}
\subsection{Real-World Robot Experiments}

\vspace{-0.5em}

To validate \method's effectiveness in a real-world environment, we conduct experiments on a DROID~\citep{khazatsky2024droid} platform featuring a 7-DoF Franka Research 3 robot arm.
We employ $\pi_{0.5}$~\citep{intelligence2025pi05} as the base generalist policy, Gemini-2.5-Flash \cite{team2025gemini} for semantic grounding, VGGT~\cite{wang2025vggt} for geometry reconstruction, and HaPTIC~\cite{ye2025predicting} for human trajectory extraction.
We evaluate nine experimental configurations across three guidance modalities, comparing \method against the base policy and one specialized method designed to solve that task. Appendix Sec.~\ref{sec:hardware} details the setup.

\noindent\textbf{Guidance Tasks:} We design a challenging suite of tasks inspired by prior works, including three different tasks for each guidance modality. For each task, an experiment is repeated across ten different trials, each with varying environmental conditions (e.g., object placement, object type). For \CA, the three tasks are avoiding collision with a static barrier (``static"), with a suddenly appearing obstacle (``dynamic"), and reacting to human blockage (``reactive"), following \cite{dalal2024neural, yang2025deep}. For \SG, the robot must reason about the given instruction to solve \textit{Multi-Choice} tasks following \citep{li2025hamster, cheang2025gr}. For the ``Picture" tasks, the robot must pick up an object and place it in a bowl closer to a picture of a named celebrity. For the ``Objects" task, five different objects are present, and the robot must pick up the correct object as instructed. For the ``Trash Cans" task, the object must correctly place the soda can in a recycled bin as opposed to a compost bin. For human \TJ~ imitation, common articulation manipulation tasks including opening a drawer, a cabinet, and a miniature oven are considered. Please refer to Appendix Sec.~\ref{sec:hardware} for details.

\begin{figure}
    \centering
    \includegraphics[width=1.0\linewidth]{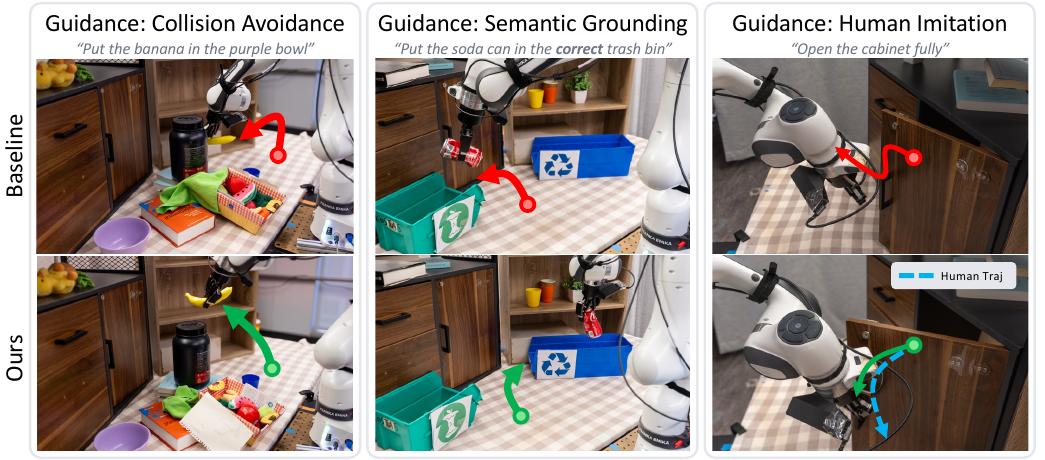}
    \vspace{-2em}
    \caption{\textbf{Real-world Qualitative Results}: Our method can handle different types of guidance including 3D awareness (left), semantic reasoning (middle), and trajectory demonstration (right). The top row shows the base $\pi_{0.5}$ policy struggling to complete the tasks while the bottom row features our method's successes. Arrows illustrate the end effector trajectories.}
    \vspace{-1.5em}
    \label{fig:real_viz}
\end{figure}

\noindent\textbf{Baselines:}
For \CA, we compare against a post-hoc strategy~\cite{tolle2025towards, romer2025demonstrations} to secure VLA's generated trajectories. Specifically, with a cost function comprising an alignment term (staying closer to the initial trajectory), a goal-reaching term, and a collision avoidance term, we initialize the trajectory with the VLA's output, and iteratively optimize this cost function using a combination of stochastic sampling MPPI \cite{williams2016aggressive} and gradient-based L-BFGS \cite{liu1989limited} optimization. The implementation is based on GPU-accelerated operations in cuRobo \cite{curobo_report23}. For \SG, we compare against F3RM \cite{shen2023distilled}, a method for localizing relevant objects in the scene using distilled CLIP features. Lastly, for human \TJ~imitation, we compare against DemoDiffusion \cite{park2025demodiffusion}, a recent approach that initializes the VLA's initial noise distribution with a kinematically retargeted human trajectory.

\noindent\textbf{Results:} As demonstrated in Fig.~\ref{fig:real_viz} and Fig.~\ref{fig:realworld_results}, \method outperforms the base VLA and specialized baselines across all tasks.
For \CA, we observe that when the base VLA outputs a disastrously unsafe trajectory, it is difficult for a post-hoc method to mitigate. \method, in contrast, intervenes with guidance during the denoising process, ensuring that the generated actions are safe.
For \SG, we found that by using a VLM \cite{team2025gemini} for high-level reasoning, we can guide the robot to perform manipulation requiring more nuanced reasoning or world knowledge (e.g., ``where to discard the soda can", ``put the object in the bowl next to celebrity X") where the base VLA struggles. The specialized baseline F3RM, which uses the less expressive CLIP features, similarly fails in complex reasoning. Additionally, this method requires human demonstrations to learn object manipulation, whereas \method inherits the base VLA's manipulation capabilities.
Lastly, for human \TJ~imitation, \method also outperforms the base VLA and the specialized method. We found two critical limitations of DemoDiffusion: 1) It provides guidance in an open-loop manner: the next guidance action chunk from the human trajectory will be provided regardless of whether the previous chunk was successfully imitated.
In contrast, \method employs a sequential matching process (Appendix Alg~\ref{alg:alignment}) to determine the optimal imitate trajectory segment adaptively. 2) Demodiffusion only influences the generative process by shaping the initial noise distribution, while \method provides denoising guidance, which is more effective as illustrated in ablation.

\noindent\textbf{Latency overhead of \method:} Real-time robot deployment demands that a method can produce action predictions without significant delays. In Fig. \ref{fig:realworld_results}, we compare the inference speed of \method against the base VLA, broken down by different components, measured on a single Nvidia RTX 5090 GPU.
In practice, the control frequency is reduced from 30Hz to around 15Hz, which is still fast enough for our policy to be reactive as illustrated by the ``reactive" experiment.

\begin{figure}
    \centering
    \includegraphics[width=\linewidth]{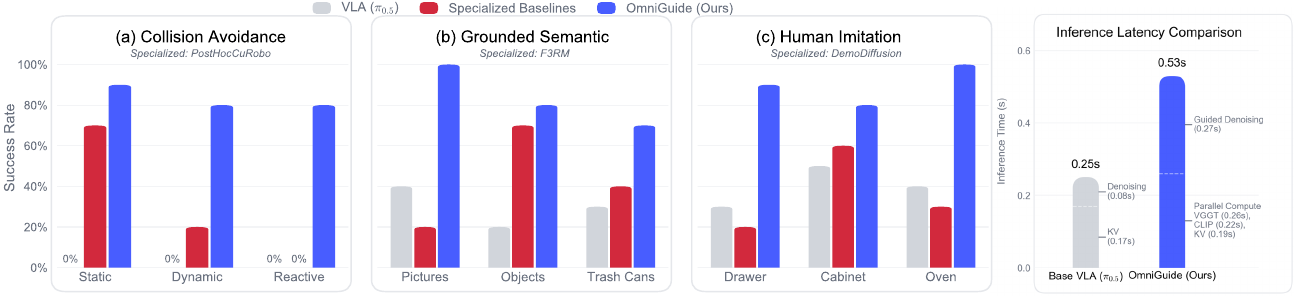}
        \vspace{-2em}
    \caption{\textbf{Real-world Experiments Quantitative Results.} \method consistently outperforms the base VLA and other specialized methods across three guidance modalities and nine tasks.}
    \label{fig:realworld_results}
    \vspace{-1.5em}
\end{figure}

\vspace{-0.5em}
\subsection{Ablations and Further Analyses}\label{sec:ablation}
\vspace{-0.5em}

\noindent\textbf{Impact of different components.} \method consists of guided noise initialization (Eq.~\ref{eqn:init}) followed by denoising guidance (Eq.~\ref{eqn:intermediate}). To quantify each component's contribution, we conduct an ablation study. As shown in Fig.~\ref{fig:sim_all}, initialization guidance improves success by 8\% and reduces collisions by 18\%, while denoising guidance yields larger gains (20\% success, 34\% collision reduction). Combined, they improve success by 26\% and reduce collisions by 46\%.

\noindent\textbf{Guidance Strength Sensitivity Analysis.} What are the trade-offs between adherence to the pretrained policy and the strength of external guidance? We investigate the system's sensitivity to the guidance weight $\lambda_C$ in the collision avoidance PnP task.
The initialization guidance is disabled to strictly isolate the guidance gradient effect. Fig.~\ref{fig:sim_all}(Mid) illustrates a clear trade-off: as guidance strength increases, the safety utility improves; however, excessively high weights eventually compete with the base policy's task objective, leading to a decline in success rate.
Crucially, the semi-logarithmic plot also shows that \method is robust to variations in hyperparameters.

\noindent\textbf{Synergy of Combined Guidance.}
One major advantage of the proposed unified spatial guidance, which existing methods have not demonstrated, is that it combines multiple guidance objectives into a single signal to steer robot actions as visualized in Fig.~\ref{fig:method_vis} (right).
We evaluate the composability of our framework in a cluttered \textit{Multi-Choice} environment, a scenario necessitating simultaneous semantic grounding and collision avoidance. As summarized in Fig.~\ref{fig:sim_all} (right), while each guidance modality proves effective in isolation, their concurrent application yields cumulative improvements. This demonstrates that the unified energy formulation can successfully coordinate heterogeneous objectives
without destructive interference. The synergy effect of different guidance weights is shown in Appendix Fig.~\ref{fig:Cartesian}.

\section{Conclusion and Limitations}

\method combines the strengths of both components: the data-driven VLA prior covering from high-level planning to contact-awareness, and the precise, task-specific, and environmental structure provided by external guidance. This combination enables controllable yet natural behavior that neither component could achieve on its own. We showed that \method improves task success and safety across guidance tasks and environments. Our results also show why pretrained VLA priors still remain essential: Guidance fields alone cannot reliably resolve kinematic mismatches in human-to-robot retargeting, avoid local minima inherent to potential-field formulations, or model complex contact dynamics.  There are several promising directions for addressing current limitations. Incorporating object-centric representations, such as hand-object geometry and articulated contact states, could improve dexterity. Other modalities such as force sensors \cite{zhang2026touchguide}, point-tracks \cite{bharadhwaj2024track2act}, AI-generated videos \cite{wiedemer2025video}, retrieved actions \cite{sridhar2025ricl}, or UMI trajectories \cite{chi2024universal} can also be incorporated as guidance in the future.

\textbf{{Limitations}} As our method shapes the prior distribution with soft guidance, there is no strict success guarantee for the objective. The guidance hyperparameters still need to be set manually. We will explore in the future how to automate the tuning of these parameters.

\textbf{Acknowledgement} The financial support from the grants NSF FRR 2220868, NSF IIS-RI 2212433, ONR N00014-22-1-2677, and NSF SLES 2331783, NSF CAREER 2239301, DARPA TIAMAT HR00112490421 is gratefully acknowledged. We thank Jinqi Luo for valuable feedback on the presentation during the review process.

\clearpage

\bibliography{references}

\clearpage
\section{Appendix}

\begin{yhblock}
\subsection{Qualitative Comparison with Baselines}
\label{app:qual_compare}
The central design choice in \method is to expose guidance through energy fields defined in 3D Cartesian space. Once a constraint can be written as an energy over the predicted robot trajectory, it can be injected into the same denoising process regardless of whether the signal comes from geometry, semantics, or demonstrations. More importantly, this energy-based interface makes guidance compositional: multiple constraints can be summed and optimized jointly, allowing a single policy sample to be shaped by safety, semantic, and imitation objectives at the same time.

Tab.~\ref{tab:qual_compare} summarizes this capability-level distinction against the specialized baselines used in our real-world experiments. F3RM provides semantic grounding by constructing a scene-specific 3D feature field, but it is not designed as a general action-generation guidance interface. cuRobo provides explicit collision-aware trajectory optimization, while DemoDiffusion provides demonstration-based initialization; however, these baselines address specific guidance types rather than a unified, compositional energy-field interface for VLA generation.

\begin{table}[h!]
\centering
\scriptsize
\setlength{\tabcolsep}{3pt}
\renewcommand{\arraystretch}{1.18}
\begin{tabular*}{0.96\linewidth}{@{\extracolsep{\fill}}lcccccc@{}}
\toprule
\textbf{Method} &
\textbf{Training-Free} &
\textbf{Safety} &
\textbf{Semantic} &
\textbf{Human demo} &
\textbf{Composition} &
\textbf{Reactive} \\
\midrule
PostHoc cuRobo & \cmark & \cmark & \xmark & \xmark & \pmark & \pmark \\
F3RM & \pmark\,{\tiny(field)} & \xmark & \cmark & \xmark & \xmark & \xmark \\
DemoDiffusion & \cmark & \xmark & \xmark & \cmark & \xmark & \xmark \\
\midrule
\textbf{\method (Ours)} & \cmark & \cmark & \cmark & \cmark & \cmark & \cmark \\
\bottomrule
\end{tabular*}
\vspace{0.35em}
\caption{\textbf{Qualitative comparison of guidance capabilities.} \cmark indicates native support, \xmark indicates no explicit support, and \pmark indicates partial support.}
\label{tab:qual_compare}
\end{table}
\end{yhblock}

\subsection{Further Study on Hyperparameters}
\textbf{Combined Guidance Synergy Study} We demonstrate the synergy effect of combined guidance under different weights, as shown in Fig.~\ref{fig:Cartesian}. We find that over-guiding on one guidance may trigger a counter-effect on both the other type and/or self, while the harmony only exists in certain regions.

\begin{figure}[h!]
    \centering
    \includegraphics[width=0.5\linewidth]{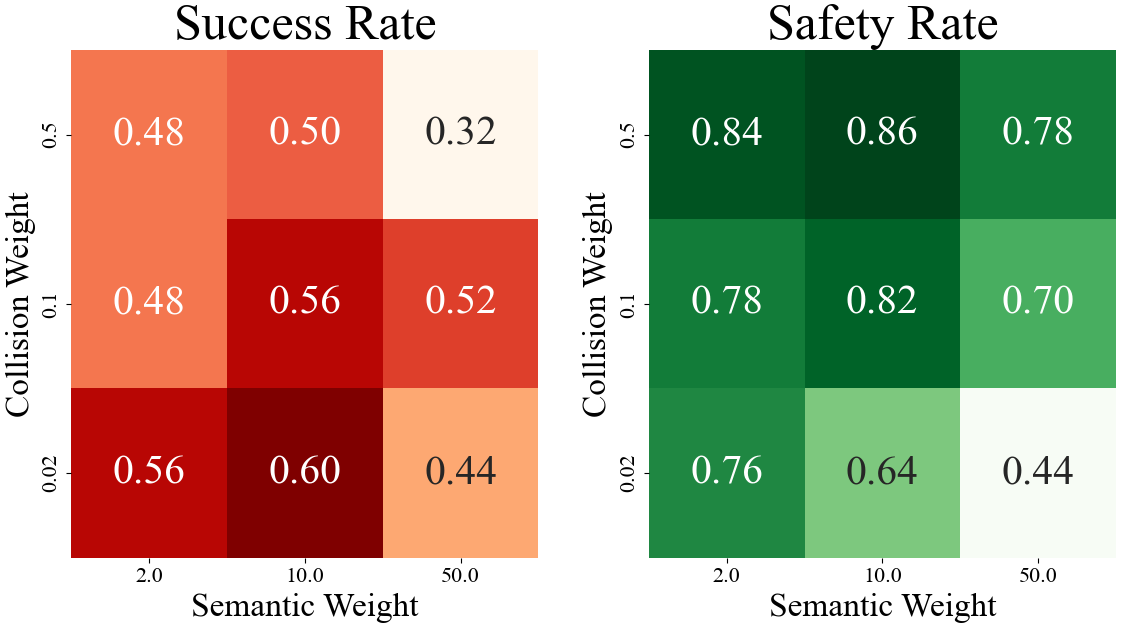}
    \caption{Synergy analysis under different guidance weights of the cluttered PnP task. There is a global weight of 0.5 finally applied to all the weights. }
    \label{fig:Cartesian}
\end{figure}

\begin{table}[h!]
    \centering
    \scriptsize
    \setlength{\tabcolsep}{3pt}
    \begin{minipage}{0.44\linewidth}
        \centering
        \begin{tabular}{l|ccccc}
        pos noise $\alpha$ & 0.0 & 0.05 & 0.1 & 0.2 & 0.5 \\
        \hline
        Succ\% & 64 & 62 & 62 & 54 & 44
        \end{tabular}
        \caption{Performance against semantic grounding noise.}%
        \label{tab:semantic_noise}
    \end{minipage}\hspace{1em}
    \begin{minipage}{0.44\linewidth}
        \centering
        \begin{tabular}{l|ccccc}
        dep noise $\alpha$ & 0.0 & 0.002 & 0.005 & 0.01 & 0.02 \\
        \hline
        Succ\% & 72 & 68 & 54 & 56 & 48 \\
        Safe\% & 48 & 52 & 60 & 66 & 74
        \end{tabular}
        \caption{Performance against depth noise.}%
        \label{tab:depth_noise}
    \end{minipage}
\end{table}

\textbf{Perception Noise Robustness} What is the system's robustness against perception noise? We add uniform noise ($U[-\alpha, \alpha]$, Tab.~\ref{tab:semantic_noise}) to semantic positions on the multi-choice task, and depth-aware noise ($D'=D+\alpha D^2N(0, 1)$, Tab.~\ref{tab:depth_noise}) to GT depthmaps on the turn sink spout task. We find our method exhibits robustness against perception noise. Meanwhile, noise in depth indicates expansion of obstacle boundaries, which improves the safety as a reasonable side effect.

\textbf{Probe Point Selection} To study the effect of probe-point selection strategy on collision-avoidance performance. We conduct experiments on the turn sink spout task with different probe point selection methods. Tab.~\ref{tab:probe} and Fig.~\ref{fig:probe} illustrate the tradeoff that more probe points can increase safety to some extent, but also hinder the main task when too many.

\begin{table}[htbp]
\hspace{-5em}
    \centering
    \renewcommand{\arraystretch}{1.2}
    \setlength{\tabcolsep}{2pt}
    \small
    \begin{minipage}{0.49\textwidth}
        \begin{tabular}{l|cccccc}
        Probe Points & 0 & 1(eef) & 1(top) & 2(eef+top) & 3(2+mid) & 5(3+LR) \\
        \hline
        Succ\% & 76 & 84 & 76 & 72 & 66 &  66  \\
        Safe\% & 8 & 14 & 40 & 48 & 58 &  62
        \end{tabular}
        \caption{Performance of different probe point selection strategies.}%
        \label{tab:probe}
    \end{minipage}

\end{table}

\begin{figure}[h!]
    \centering
    \includegraphics[width=0.5\linewidth]{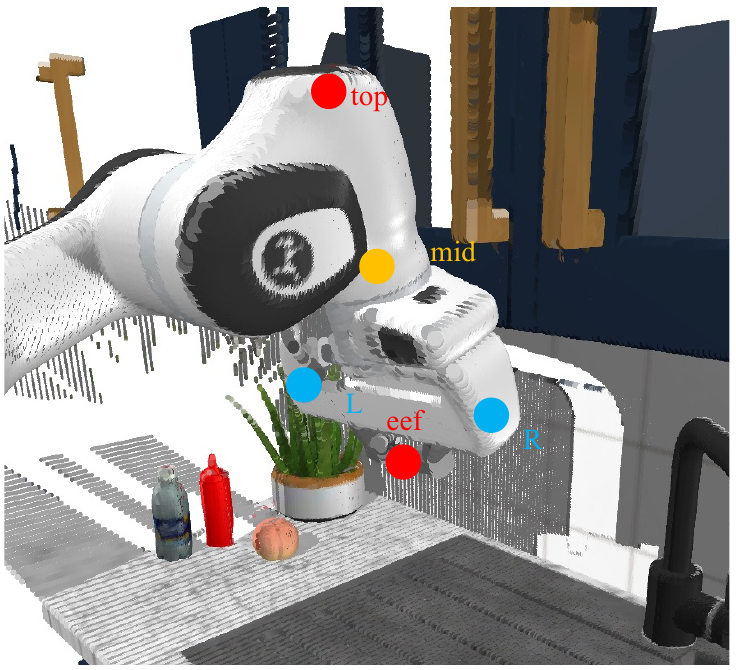}
    \caption{Probe points position illustration.}
    \label{fig:probe}
\end{figure}

\textbf{Local Minima Demonstration} We set up a cluttered scene and visualize the local minima in Fig.~\ref{fig:localmin}. The energy field is naturally flat around the local minima, while the base policy tends to produce non-static actions due to the training distribution, which perturb the robot out of the local minima. Experiments show that our method can still provide helpful signals towards task success and safety.

\begin{figure}[h!]
    \centering
    \includegraphics[width=1.0\linewidth]{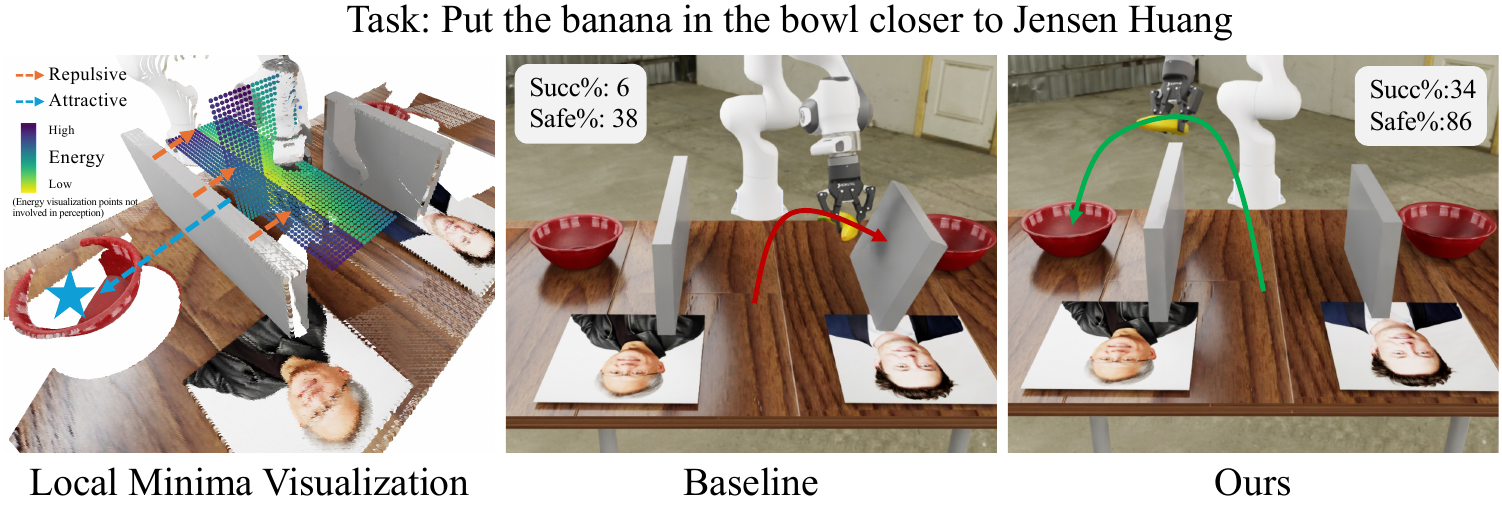}
    \caption{Local minima experiment setup and results.}
    \label{fig:localmin}
\end{figure}

\subsection{Simulation Experiment Setup}
\label{sec: sim_setup}
We choose the NVIDIA Gr00t N1.6 \cite{gr00tn1_2025} model as our base policy for simulation experiments. The policy performs flow-matching denoising in the latent space and decodes the latent into action chunks that contain relative end effector positions $\{ \Delta\mathbf{x}_i\}$, and orientations $\{ \Delta \mathbf{r}_i\}$ in the axis-angle convention, whose corresponding rotation matrix is computed as $R(\Delta \mathbf{r}_i)$. We use the Euler method to estimate the absolute robot positions and orientations as
\begin{align}
    & \mathbf{x}_i = \mathbf{x}_{i-1} + \mathbf{\gamma}_x \odot \Delta \mathbf{x}_i \\
    & R_i = R( \gamma_r \odot \Delta \mathbf{r}_i) R_{i-1}
\end{align}
where $\gamma_x$ and $\gamma_r$ are dynamic parameter vectors estimated from simulator properties, and $\odot$ denotes element-wise product. We set $\gamma_x = [0.011, 0.011, 0.02]$ and $\gamma_r = [0.15, 0.15, 0.15]$.

For the PnP-series tasks, we configure the collision barrier distance (Sec.~\ref{sec:collision}) as $d=0.15$ meter, with collision, semantic, and trajectory guidance weights set to $\lambda_C = 0.05$, $\lambda_S = 5.0$, $\lambda_T = 5.0$ respectively, alongside an auxiliary orientation term weighted at $0.02\lambda_S$. For the turn sink spout task, we set the collision guidance weight $\lambda_C = 0.3$.

The planner-generated trajectory and \TJ~ guidance are visualized in Fig.~\ref{fig:trajguide}. We define the position above the object and the target container locations as planner goals. This guidance may be suboptimal because the cuRobo-produced trajectories may be too close to the safety boundaries, and the robot merely escapes collisions.

\begin{figure}
    \centering
    \includegraphics[width=1.0\linewidth]{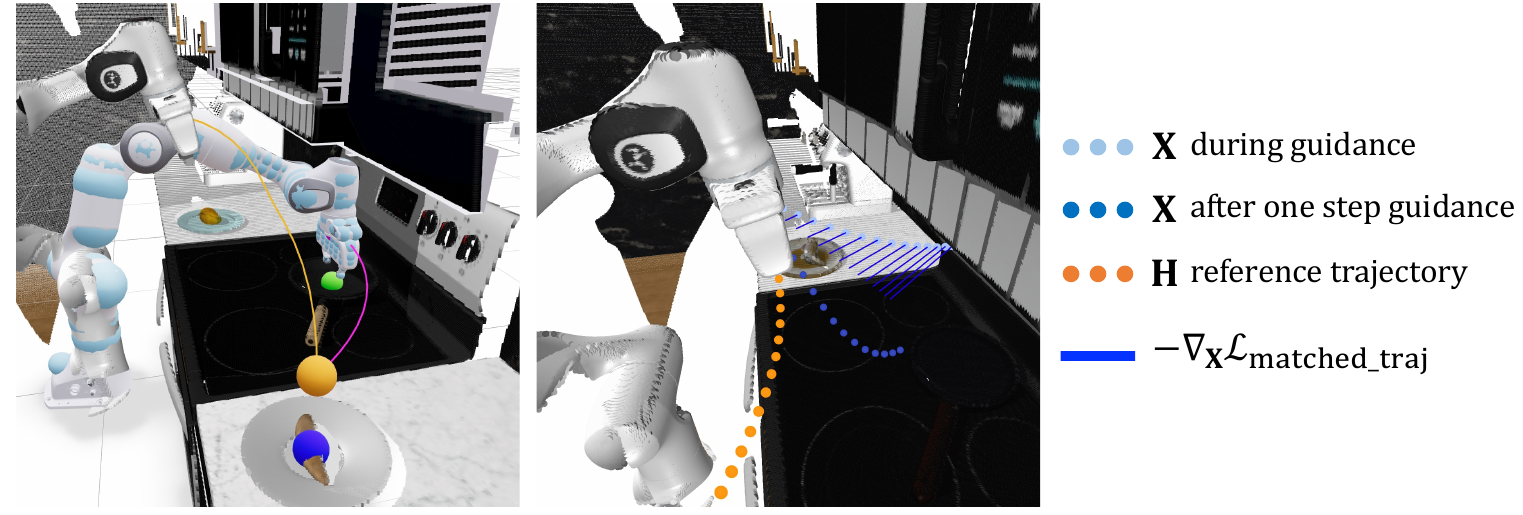}
    \vspace{-1em}
    \caption{\textbf{Left:} The yellow and magenta lines are the collision-free and feasible trajectory planned by cuRobo\cite{curobo_report23}. \textbf{Right:} Visualization of the sequential matching guidance effect.}
    \label{fig:trajguide}
    \vspace{1em}
\end{figure}

\begin{figure}[h]
    \centering
        \includegraphics[width=0.8\linewidth]{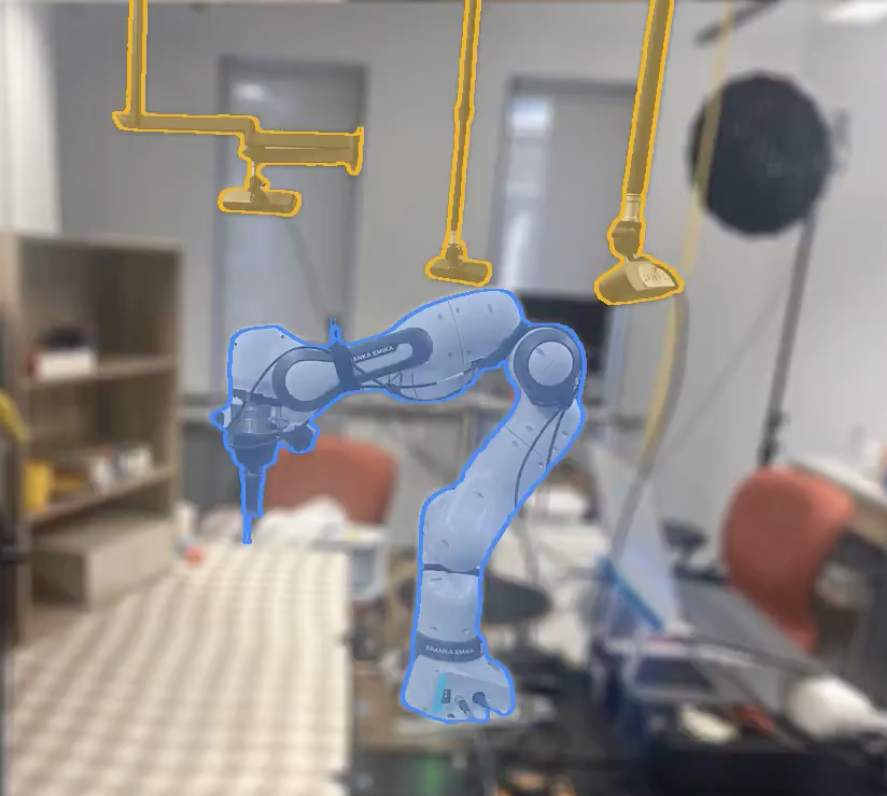}

        \caption{\textbf{Hardware Setup.} We use a standard 7-DOF Franka Emika Panda arm with a Robotiq 2-fingered gripper and three Zed 2 stereo cameras to provide left, right and overhead views for 3D reconstruction.}
    \label{fig:hardware}
\end{figure}

\begin{figure}[h]
    \centering
    \begin{minipage}{0.48\textwidth}
        \centering
        \includegraphics[width=\textwidth]{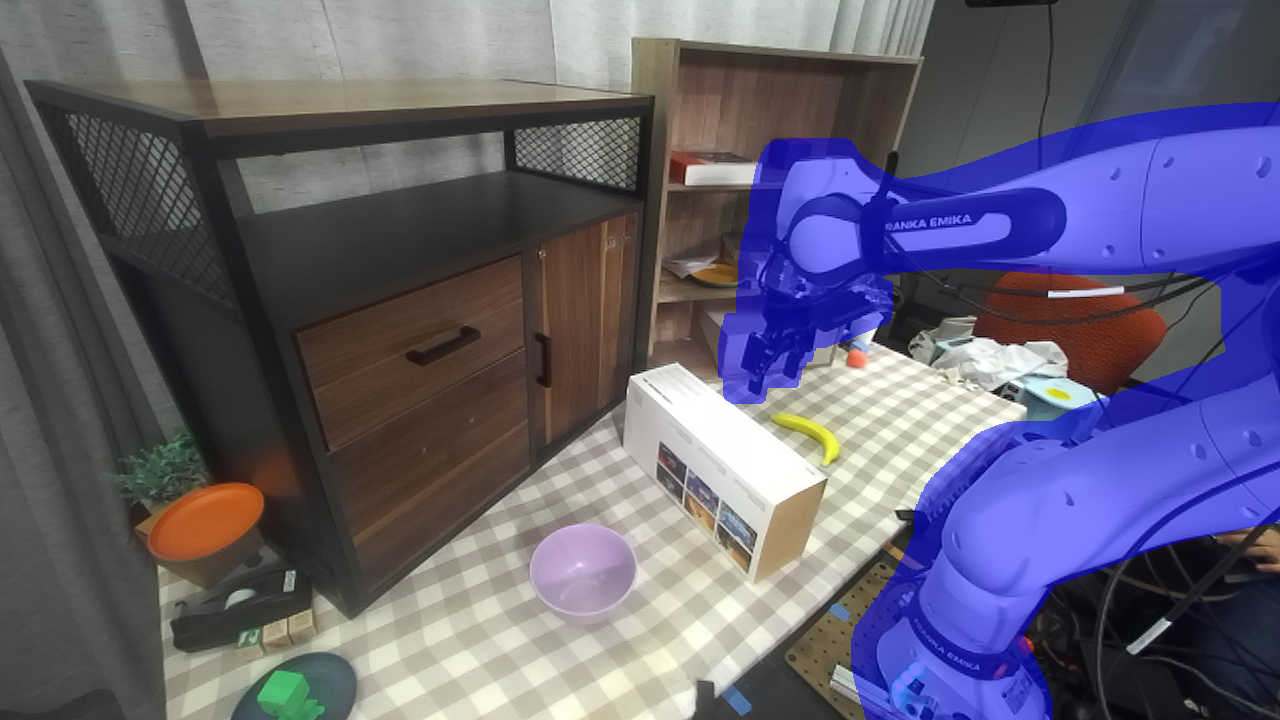}
    \end{minipage}
    \begin{minipage}{0.48\textwidth}
        \centering
        \includegraphics[width=\textwidth]{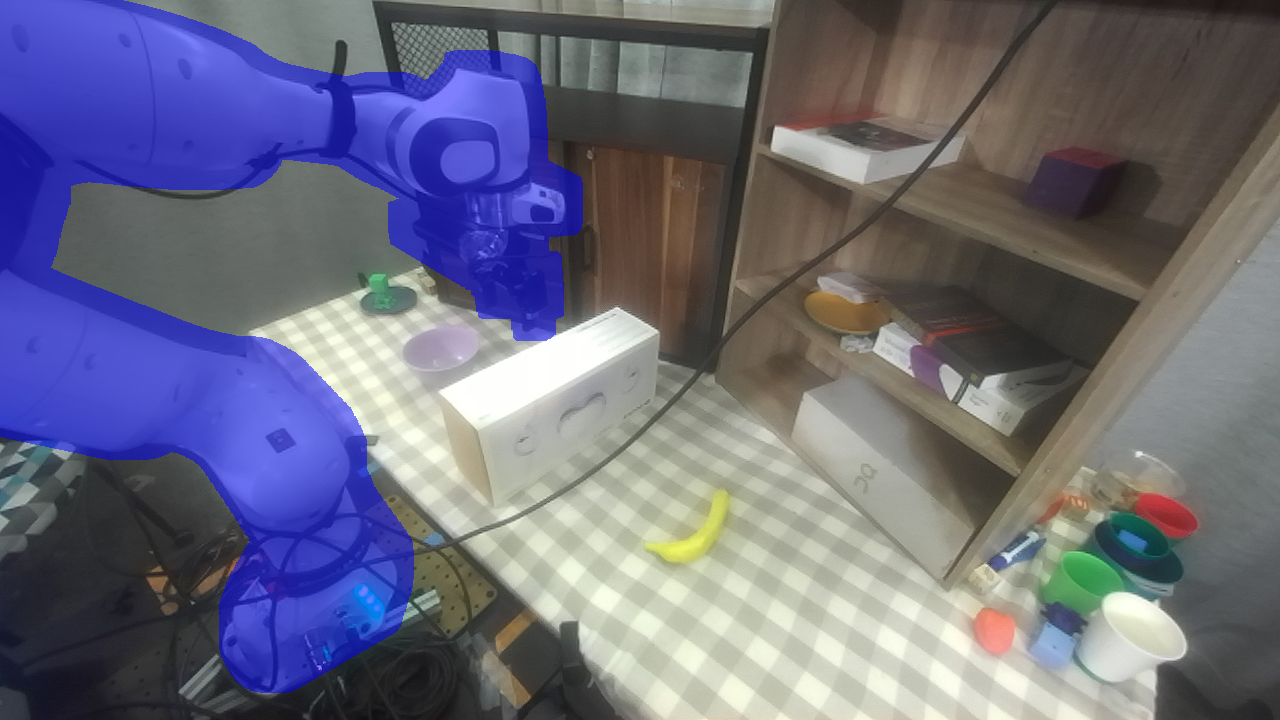}
    \end{minipage}

    \begin{minipage}{0.48\textwidth}
        \centering
        \includegraphics[width=\textwidth]{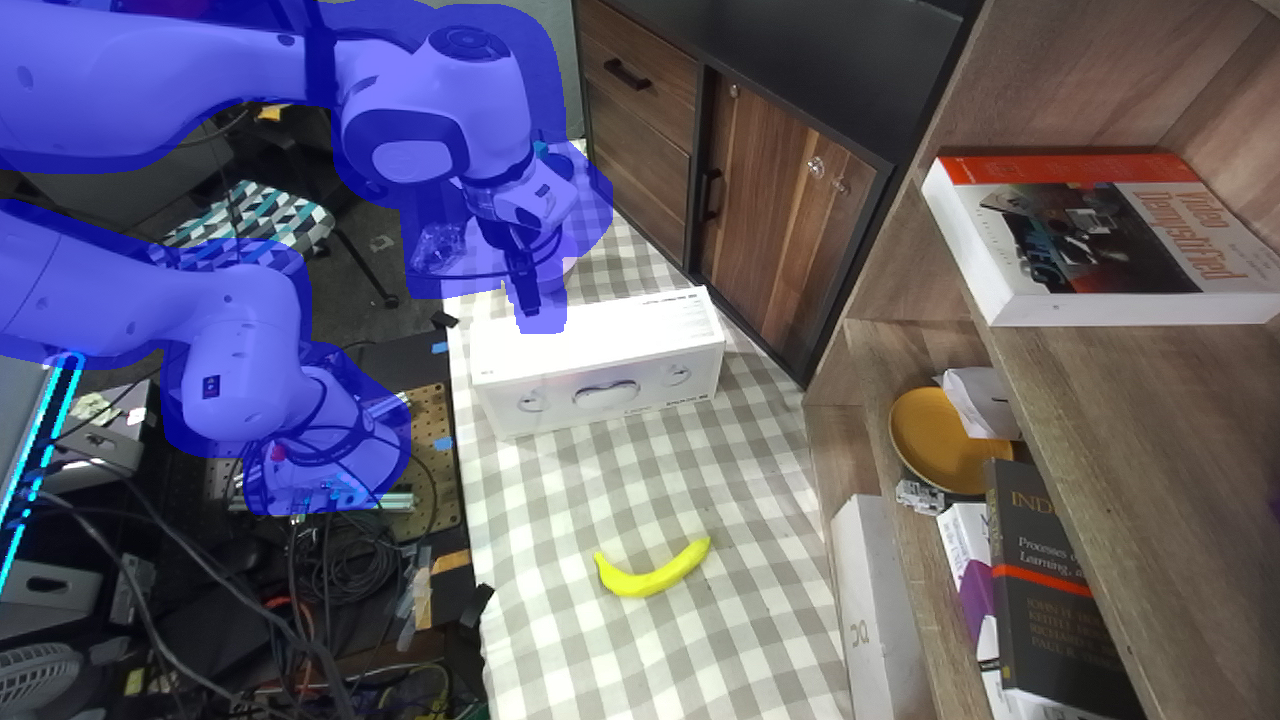}
    \end{minipage}
    \begin{minipage}{0.48\textwidth}
        \centering
        \includegraphics[width=\textwidth, trim=0 20 0 20, clip]{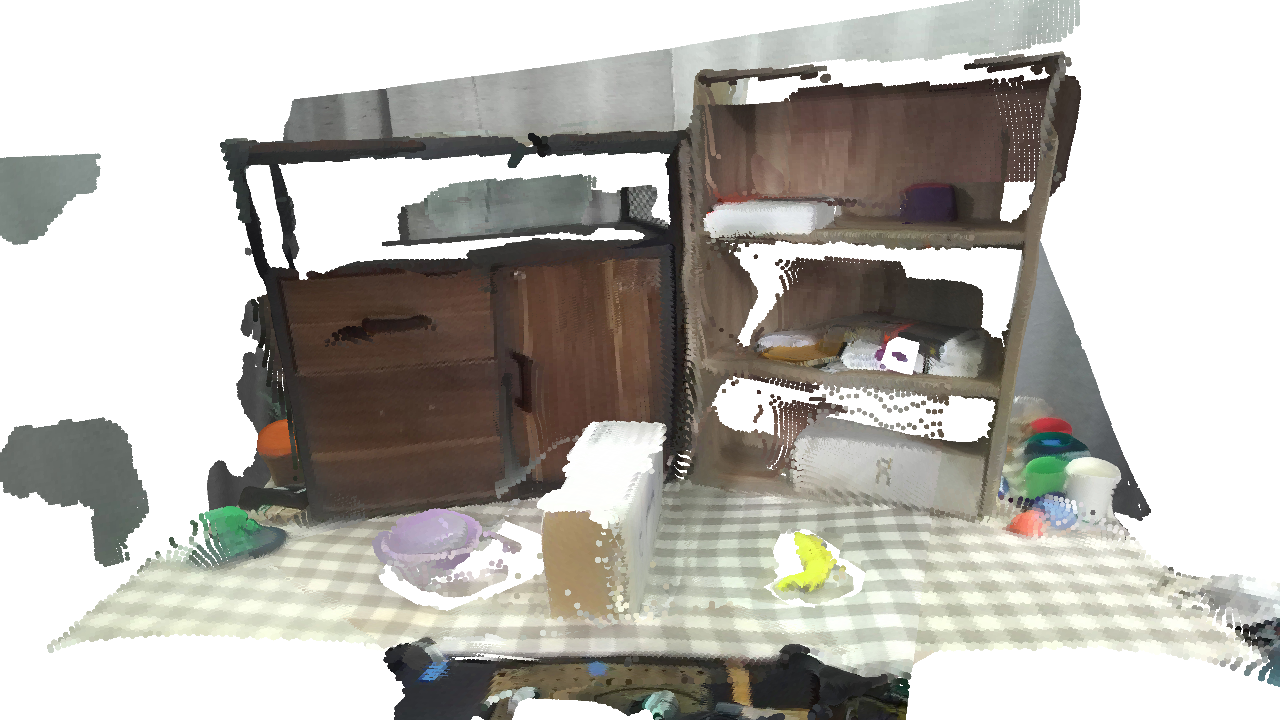}
    \end{minipage}

    \caption{\textbf{VGGT 3D Reconstruction}. We reconstruct  a 3D point cloud of the scene three camera views (left, right, and overhead view). We ensure that the scene has proper metric scale by comparing the ground-truth camera pairwise distance against the VGGT's predicted camera distance, and rescaling the depth accordingly. We remove  the reconstruction of the robot from the scene point-cloud. We denote with shaded blue the rendering of the robots in simulation using Pybullet \cite{benelot2018} and the calibrated camera extrinsics, and superimpose it on images.  \label{fig:camera}}
    \vspace{-0.5cm}
\end{figure}

\subsection{Real-World Experiment Setup} \label{sec:hardware}

We followed the DROID robot setup~\citep{khazatsky2024droid}, which consists of a 7 DoF Franka Emika Panda Robot Arm, a Robotiq 2F-85 parallel-jaw gripper, a wrist-mounted ZED Mini RGB-D camera and two side-mounted ZED 2 stereo cameras on the left and right side view. We additionally mounted an overhead camera to provide more coverage for 3D reconstruction (Fig.~\ref{fig:hardware}).
\vspace{0.5em}

\textbf{VGGT Metric-Scale Reconstruction:} We use VGGT \cite{wang2025vggt} to construct a pointmap from each of the RGB view, resulting in a 3D point-cloud.
However, VGGT predicted depth is scale-invariant. To achieve the metric-accurate reconstruction required for precise manipulation, we perform a scale-alignment procedure. We define the ground-truth Euclidean distance between a calibrated camera pair as $d_{gt}$. Given the VGGT-predicted camera extrinsics, we compute the relative predicted distance $d_{pred}$. The point cloud depth is then rescaled by the ratio $\kappa = d_{gt} / d_{pred}$. This ensures the resulting point cloud aligns with the physical dimensions of the robot's workspace as shown in Fig.~\ref{fig:camera}.

\vspace{0.5em}

\textbf{HaPTIC Human Demonstration Extraction}:
We extract the human thumb positions as the robot end-effector guidance target. In Fig~\ref{fig:haptic}, the blue mesh shows the full HaPTIC extracted hand, and the magenta point is the reference trajectory.
\begin{figure}[h!]
    \centering
    \includegraphics[width=0.8\linewidth]{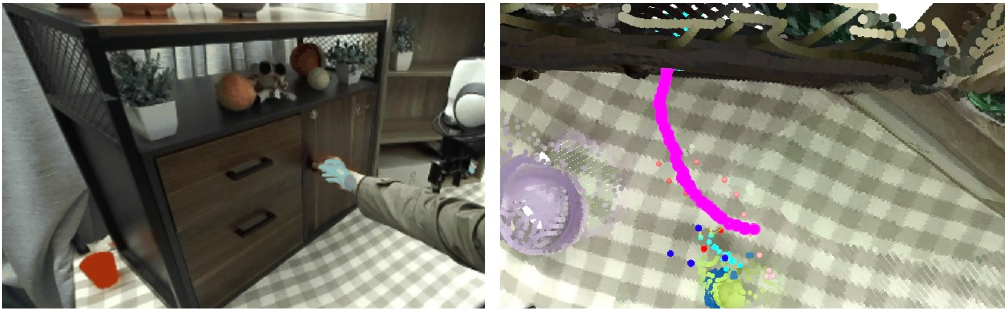}
    \caption{HaPTIC hand poses visualization.}
    \label{fig:haptic}
\end{figure}

\textbf{Task-relevant Region Detection and Filtering}: To facilitate collision-free planning, we must distinguish between the static environment, the robot's own body, and task-relevant objects. For the robot body, we use Pybullet simulation \cite{benelot2018} to render the robot URDF synchronized with the robot's calibrated extrinsics, and superimpose the resulting mask onto the camera frames (Fig.~\ref{fig:camera}). This allows us to remove the robot body from the aggregate point cloud.

For object removal, we employ a variant of the CLIP model \cite{qiu2024feature, zhou2022extract} to extract dense patch-level visual feature embeddings of each camera frame. These features are projected into 3D space using the VGGT pointmaps. Next, for each 3D point $i$ in the reconstructed point cloud, we compute the task relevance score
\begin{equation} s_i = \frac{\mathbf{e}_{i} \cdot \mathbf{e}_{\text{text}}}{|\mathbf{e}_{i}| |\mathbf{e}_{\text{text}}|} \end{equation}
using the cosine similarity between the visual embedding $e_i$ and the task instruction $e_{\text{text}}$. The 3D points scoring in top 95\% percentile are marked as task relevant and removed from the collision point cloud. A visualization of the heatmap of the CLIP similarity and the resulting filtered point-cloud is provided in Fig.~\ref{fig:clip}.

All of our inference including VGGT, CLIP, guidance computation and denoising is run using a single Nvidia RTX 5090 GPU.

\begin{figure}[!h]
    \centering
    \begin{minipage}{0.48\textwidth}
        \centering
        \includegraphics[width=\textwidth]{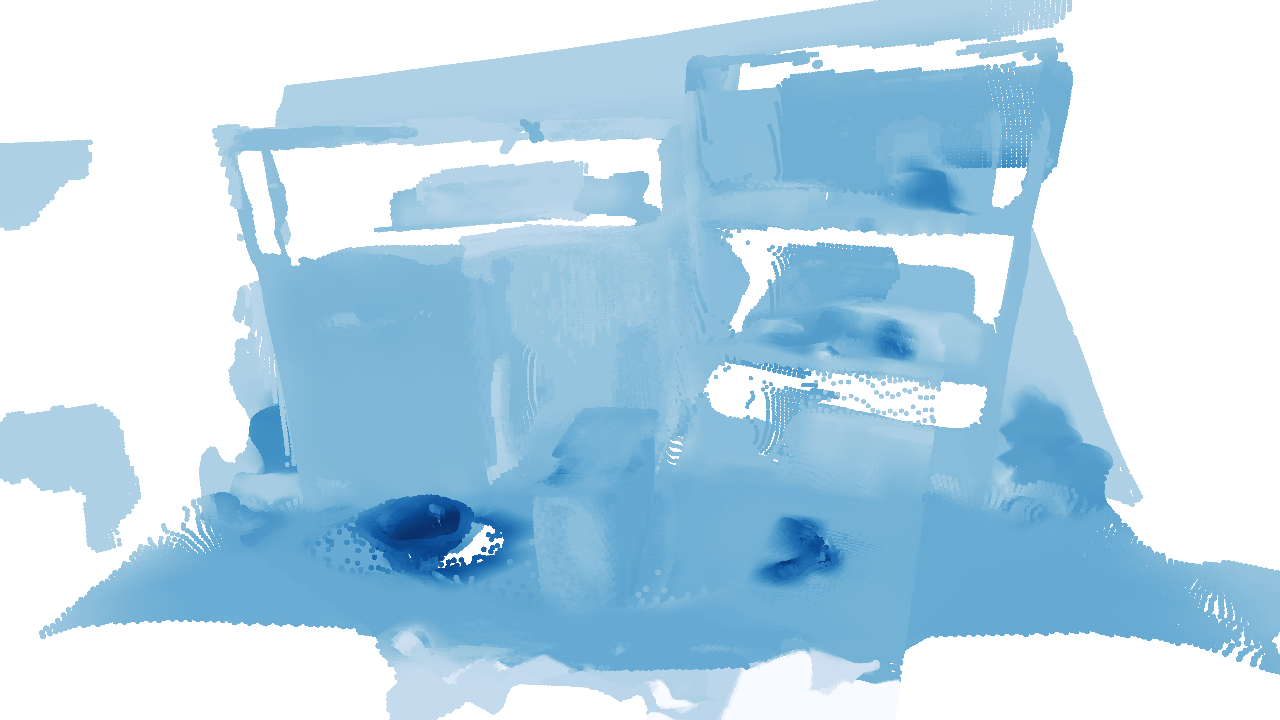}
    \end{minipage}
    \hspace{2mm}
    \begin{minipage}{0.48\textwidth}
        \centering
        \includegraphics[width=\textwidth]{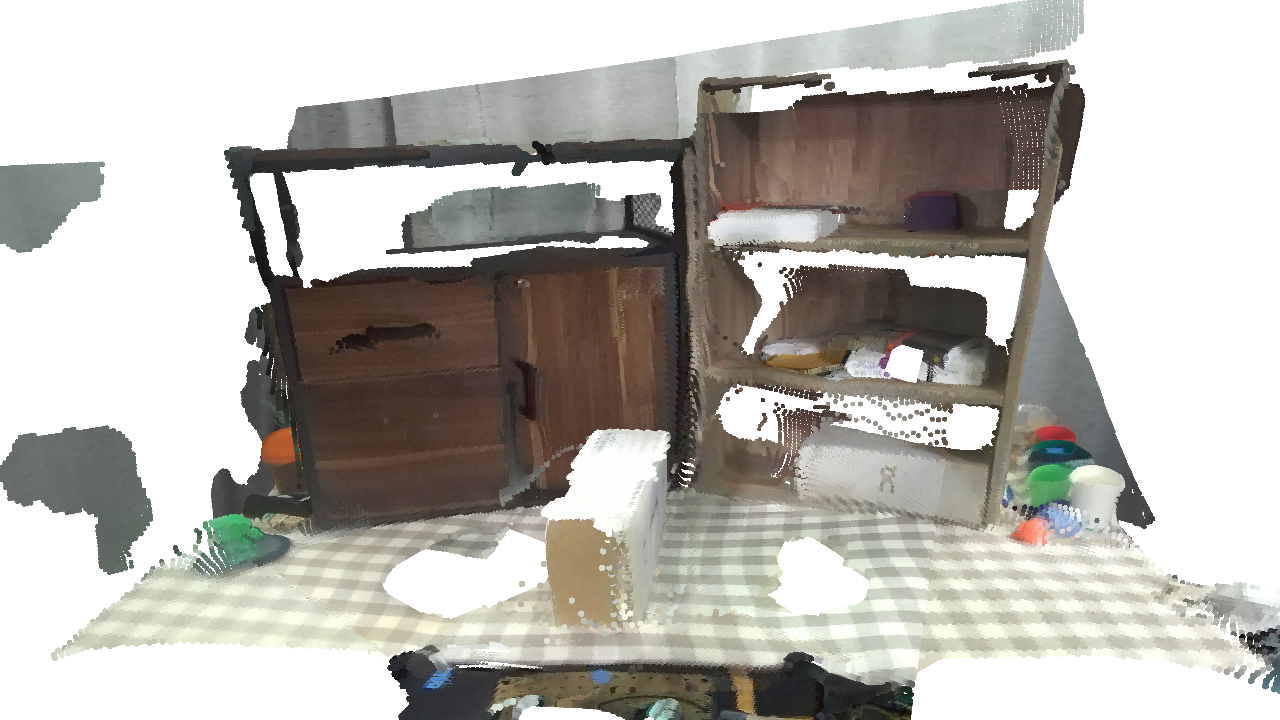}
    \end{minipage}
    \caption{\textbf{Detecting Task-relevant Regions}. We use patch-level CLIP to compute 3D visual features of the point cloud. We then compute the cosine similarity of each 3D point against the language instruction. The left image is the heatmap of similarities for the task ``put the banana in the purple bowl". Task-relevant regions are defined to be the 3D points scoring in the top 95\% percentile in CLIP similarity. The right image shows the filtered point cloud where the banana and the purple bowl were correctly removed using our procedure.\label{fig:clip}}
\end{figure}

\textbf{VLM Pointing Guidance:}
To provide grounded semantic guidance, we leveraged a capable VLM, Gemini-2.5-Flash \cite{comanici2025gemini}, to point to the relevant target in the scene (Fig.~\ref{fig:pointing}). Specifically, the VLM first provide a list of the 2D coordinates of all interesting points in the scene. Then, for each 2D candidate, we visually highlight the point and query the VLM whether it is relevant to a given task. At the end, we keep one 2D point. We then lift this 2D coordinate to 3D via VGGT predicted pointmap. This provides a semantic target for guiding VLAs in tasks requiring high-level reasoning.

\begin{figure}[!h]
    \centering
    \begin{minipage}{0.48\textwidth}
        \centering
        \includegraphics[width=\textwidth]{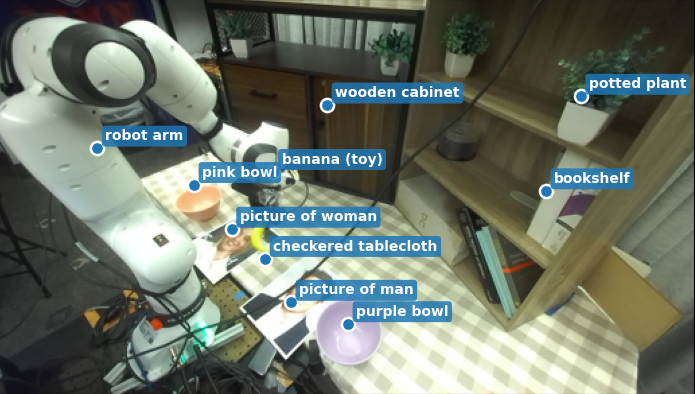}
    \end{minipage}
    \hspace{2mm}
    \begin{minipage}{0.48\textwidth}
        \centering
        \includegraphics[width=\textwidth]{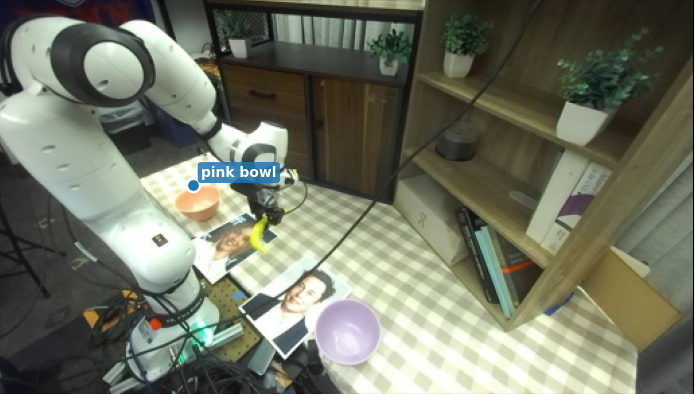}
    \end{minipage}

        \caption{\textbf{VLM Pointing Guidance}. (Left) First, the VLM points to all interesting points in an image. (Right) Then, we visually highlight each point in the image, and ask the VLM whether that point is relevant to a given task. One point is kept at the end. In this example, the VLM successfully localized the relevant object for the task of putting the object in a bowl closer to a named celebrity.\label{fig:pointing}}

\end{figure}

\subsection{Baseline Implementation Details}
\textbf{PostHoc cuRobo}: The base VLA policy generates an action chunk \(
\mathbf{A}_{t:t+H} = [\mathbf{a}_t, \mathbf{a}_{t+1}, \dots, \mathbf{a}_{t+H-1}]
\) where $H$ is the chunk size. The robot typically only executes a portion $h$ of such chunk. For our $\pi_{0.5}$ setup, $H=15$ and $h=8$. The goal of this baseline is thus to transform the effective chunk \(
\mathbf{A}_{t:t+h}\) into a collision-free trajectory $Q^*$. This is done by optimizing a given cost function by stochastic sampling  MPPI \cite{williams2016aggressive} and gradient-based optimization L-BFGS \cite{liu1989limited}. We use cuRobo default values of $2$ iterations of MPPI followed by $100$ iterations of L-BFGS. MPPI aggregates $25$ candidates per seed. The optimization objective $J(\mathcal{Q})$ minimizes the deviation from the VLA's reference actions while strictly penalizing collisions and joint limit violations:
\begin{equation}
\begin{aligned}
    \mathcal{Q}^* = \arg \min_{\mathcal{Q}} \sum_{k=t}^{t+h} \Big( & w_{\text{align}} \|\mathbf{q}_k - \mathbf{q}_k^{\text{ref}}\|^2 \\
    & + w_{\text{coll}} \mathcal{C}_{\text{sdf}}(\text{FK}(\mathbf{q}_k)) \\
    & + w_{\text{bound}} \mathcal{C}_{\text{bound}}(\mathbf{q}_k) \Big) + w_{\text{goal}} J_{\text{goal}}(\mathbf{q}_h)
\end{aligned}
\end{equation}
where $\mathbf{q}k^{ref}$ is the joint-space reference derived from the VLA action $\mathbf{a}_k$, $ \mathcal{C}_{\text{sdf}}$ represents the collision cost calculated via a voxelized Signed Distance Field (SDF) queried by robot collision spheres, $ \mathcal{C}_{\text{bound}}$ enforces physical joint and velocity limits, $J_{\text{goal}}$ measures how closely the final cartesian coordinate matches the goal where joint-to-cartesian conversion was done using forward kinematics. To ensure safety in complex environments, we tune $ w_{\text{coll}}=1e7, w_{\text{align}} = 100$, prioritizing collision avoidance over strict adherence to the VLA's initial plan. $w_{\text{bound}}=0.1$ is kept at the cuRobo's default value.

\vspace{0.5em}

\textbf{F3RM}: We followed the setup in F3RM \cite{shen2023distilled}. The robot holds a Zed camera in its gripper and scan the scene. CLIP features and then computed and lifted to 3D via aligned Colmap \cite{schoenberger2016sfm}. During scanning, we also saved the end-effector poses and use these labels to align Colmap predictions to metric depth by predicting a scale by simple RANSAC \cite{fischler1981random} and a rigid registration between the end-effector points and Colmap predicted points via the Kabsch algorithm \cite{lawrence2019purely}.

\vspace{0.5em}

\textbf{DemoDiffusion}: We followed the setup in DemoDiffusion \cite{park2025demodiffusion} closely with one modification. We replace the original HaMer \cite{pavlakos2024reconstructing} system with the more recent HaPTIC method \citep{ye2025predicting} for 3D human motion capture and reconstruction. Note that our method, \method, also use HaPTIC for the human imitation guidance experiments.

\subsection{Mathematical Details}
\begin{theorem}
The integral value $Z = \int_\Omega \mathrm{SDF}_O(\mathbf{x}) \, d\mathbf{x}$ over ${\Omega}$ is finite on the domain ${\Omega} = \{\mathbf{x} \in \mathbb{R}^3 \mid 0 < \mathrm{SDF}_O(\mathbf{x}) \le d\}$, where $d\in(0,\infty)$.
\end{theorem}
\begin{proof}
$$Z = \int_\Omega \text{SDF}_O(\mathbf{x}) \, d\mathbf{x} \le \int_\Omega d \, d\mathbf{x} = d \cdot V(\Omega)$$
where $V(\Omega)$ is the volume of the finite domain.
\end{proof}

\label{app:math_details}
\paragraph{SDF collision likelihood.}
\label{app:sdf_assumption}
For collision avoidance, the reconstructed scene geometry is represented as a point cloud or surface observation rather than as a watertight solid object. In this setting, query points with $\mathrm{SDF}_O(\mathbf{x})<0$ are rarely observed in practice. We therefore define the collision guidance over the bounded risk region $\Omega=\{\mathbf{x}\in\mathbb{R}^{3}\mid 0<\mathrm{SDF}_O(\mathbf{x})\le d\}$ and use a normalized distance likelihood $p_C(\mathbf{x})=\mathrm{SDF}_O(\mathbf{x})/Z$, where $Z=\int_{\Omega}\mathrm{SDF}_O(\mathbf{x})\,d\mathbf{x}$. This likelihood assigns lower energy to points farther from obstacles within the risk region, resulting in a repulsive field whose strength increases near contact.

\begin{yhblock}
\paragraph{Gaussian attractive guidance.}
\label{app:gaussian_guidance}
For semantic grounding, the VLM and depth reconstruction provide a point estimate $\mathbf{x}^{*}$ for the task-relevant object or affordance. We model the probability that a predicted end-effector position $\mathbf{x}$ grounds successfully to this target with an isotropic Gaussian likelihood, $p_S(\mathbf{x})=\mathcal{N}(\mathbf{x}\mid\mathbf{x}^{*},\sigma_S^2\mathbf{I})$. This assumption yields the quadratic attractive energy used in the main text, with $\sigma_S$ controlling the spatial tolerance around the grounded target. The same likelihood form is used for human-demonstration attractors, replacing the semantic target $\mathbf{x}^{*}$ with matched demonstration points $\mathbf{h}_i^{*}$.

\end{yhblock}

\subsection{Implementation of the Sequential Trajectory Matching Method}
\label{sec: matching}
The pseudo code for the sequential trajectory matching is described in Alg~\ref{alg:alignment}. Since the robot horizon $H$ and the reference demonstration length $N$ may differ, guidance is applied only to the matched subset \(\mathcal{M}\).

\begin{algorithm}[h!]
\caption{Monotonic Trajectory Alignment}
\label{alg:alignment}
\begin{algorithmic}[1]
\Require Predicted Robot Trajectory $\mathcal{X} = (\mathbf{x}_0, \dots, \mathbf{x}_{H-1})$, where $\mathbf{x} \in \mathbb{R}^3$
\Require Reference Trajectory $\mathcal{H} = (\mathbf{h}_0, \dots, \mathbf{h}_{N-1})$, where $\mathbf{h}_k \in \mathbb{R}^3$
\Ensure Matched Correspondence Set $\mathcal{M}$

\State \textbf{Initialize:} $\mathcal{M} \leftarrow \emptyset$, start index $k_{\text{curr}} \leftarrow 0$

\For{$t = 0, \dots, H-1$}
    \State \Comment{Find nearest neighbor in the remaining horizon}
    \State $k^* \leftarrow \operatorname*{arg\,min}_{k \ge k_{\text{curr}}} \| \mathbf{x}_t - \mathbf{h}_k \|_2$

    \State $\mathcal{M} \leftarrow \mathcal{M} \cup \{(\mathbf{x}_t, \mathbf{h}_{k^*})\}$ \Comment{Store the matched pair}

    \State $k_{\text{curr}} \leftarrow k^*$ \Comment{Pruning for temporal monotonicity}

    \If{$k_{\text{curr}} = N-1$}
        \State \textbf{break} \Comment{Reference trajectory exhausted}
    \EndIf
\EndFor

\State \Return $\mathcal{M}$
\end{algorithmic}

\end{algorithm}

\subsection{Orientation Guidance}
Given the current end effector Cartesian position $\mathbf{x}$ and the target position $\mathbf{x}^*$, we can compute the normalized target gripper orientation as
\begin{equation}
    \mathbf{r}^*= \frac{\mathbf{x}^*-\mathbf{x}}{||\mathbf{x}^*-\mathbf{x}||}
\end{equation}
We define the canonical orientation vector of the gripper in its local frame as $\mathbf{g}\in\mathbb{R}^3, ||\mathbf{g}||=1$. For example, the Franka arm gripper has the orientation vector $\mathbf{g}=[0, 0, 1]$ pointing to its z-axis. After the policy predicts the gripper orientation, we convert it into a rotation matrix $R$ and compute the current gripper orientation as $\mathbf{r} = R\mathbf{g}$. Now the orientation guidance energy can be expressed as
\begin{equation}
    \mathcal{L}_o(\mathbf{r}) = -\log p_R(\mathbf{r}) = \frac{||\mathbf{r}-\mathbf{r}^*||^2}{2\sigma_R^2} + \text{Const.}
\end{equation}

\end{document}